%% file: main.tex
\documentclass{article}
\usepackage[preprint,nonanonymous]{neurips_2026}
\makeatletter
\renewcommand{\@noticestring}{}
\makeatother
\usepackage[utf8]{inputenc}
\usepackage[T1]{fontenc}
\usepackage{microtype}
\usepackage{amsmath,amssymb,amsthm}
\usepackage{booktabs}
\usepackage{graphicx}
\usepackage{placeins}
\usepackage{multirow}
\usepackage{enumitem}
\usepackage{xcolor}
\usepackage[colorlinks=true,allcolors=blue]{hyperref}

\title{Learning Less Is More: Premature Upper-Layer Attention Specialization Hurts Language Model Pretraining}

\author{
Jinchang Zhu$^{1,a}$, Jindong Li$^{1}$, Yuwen Hao$^{1}$, Chengyu Zou, Rong Fu, Menglin Yang$^{1,\dagger,b}$ \\
$^{1}$The Hong Kong University of Science and Technology (Guangzhou) \\
$^{a}$jzhu997@connect.hkust-gz.edu.cn \quad $^{b}$menglinyang@hkust-gz.edu.cn \\
$^{\dagger}$Corresponding author.
}

\newtheorem{theorem}{Theorem}

\newtheorem{corollary}{Corollary}
\newcommand{\qk}{Q/K}

\newcommand{\probe}{upper-\qk{} learning-rate intervention}

\newcommand{\od}{\odot}

\definecolor{myblue_1}{RGB}{70, 90, 110}
\definecolor{myblue_2}{RGB}{242, 248, 253}
\definecolor{darkblue}{RGB}{0,45,110}

\usepackage{hyperref}
\hypersetup{
    colorlinks=true,
    linkcolor=blue,
    filecolor=magenta,      
    urlcolor=blue,
    citecolor=darkblue
}       % 显示颜色，为了方便编辑和检查。最后交稿时注释掉。(可以带颜色)

\usepackage[most]{tcolorbox}
\newtcolorbox{theorembox}{
  enhanced,
  breakable,
  colback=myblue_2,
  colframe=myblue_1,
  boxrule=0.8pt,
  arc=2mm,
  left=1.2mm,
  right=1.2mm,
  top=1mm,
  bottom=1mm
}

\begin{document}
\maketitle

\input{sections/00_abstract}

\input{sections/01_introduction}

\input{sections/03_premature_upper_attention_specialization}
\input{sections/04_architecture_gated_ffns}
\input{sections/05_experimental_setup}
\input{sections/06_results}
\input{sections/07_theory_summary}
\input{sections/08_discussion}
\input{sections/09_limitations}

\input{sections/10_conclusion}

\clearpage

\bibliographystyle{plainnat}
\bibliography{references}

\clearpage

\appendix

\input{sections/02_related_work}
\input{sections/07_theory}
\input{sections/a1_direct_ffn_pathway_evidence}
\input{sections/a2_entropy_floor_control}
\input{sections/a3_global_learning_rate_control}
\input{sections/a3b_release_rule_robustness}
\input{sections/a4_alpha_sweep}
\input{sections/a5_scale_check}
\input{sections/a6_locality_diagnostic}
\input{sections/a7_reproducibility_compute}
\input{sections/a8_LLM-Usage-Statement}
\end{document}

%% file: sections/00_abstract.tex
\begin{abstract}

A causal-decoder block is hierarchical: lower layers build the residual basis that upper layers attend over. We identify a failure mode of this hierarchy in GPT-style pretraining: upper layers can commit to sharp, causally relied-upon attention patterns before the lower-layer features they attend to have stabilized. We call this \textit{premature upper-layer attention specialization}, and our central claim is that this failure is real, measurable, and explains why one widely adopted architectural choice, the multiplicative gated feed-forward network, helps pretraining beyond expressivity arguments.
{First}, to establish the failure, we run a minimal controlled optimization intervention. Temporarily slowing only the query and key projections in upper layers during early training improves a 270M GPT-style decoder's final perplexity, token efficiency, and downstream accuracy, without altering any other parameter group. Mechanistic probes confirm that the intervention does not delay lower-layer routing; rather, it specifically prevents upper attention from collapsing onto an immature residual basis.
{Second}, we ask why the same intervention is nearly unnecessary in LLaMA-style blocks. Through matched ablations, we isolate multiplicative gated feed-forward networks, rather than RMSNorm or bias removal, as the component that suppresses the upstream residual writes driving the failure.
{Third}, we provide a pathwise analysis of the decoder block that unifies these two findings under a single bound. The learning-rate intervention reduces a step-size factor on this bound, while gated FFNs reduce a residual-energy factor on the same growth pathway.
Our results identify upper-layer Q/K timing as one concrete point where decoder architecture and optimization interact. Following this view, future LLM design can be guided not only by what each block computes, but by when in training it should be allowed to commit.

\end{abstract}

%% file: sections/01_introduction.tex
\section{Introduction}

Transformer pretraining is usually described through scale, data, and optimizer schedules \citep{vaswani2017attention,brown2020language,kaplan2020scaling,hoffmann2022training}. Yet the block architecture itself can determine which computations become confident early. A decoder stack is hierarchical: lower layers construct a residual basis, and upper layers use attention to compare, retrieve, and recombine information in that basis. Upper-layer query/key projections therefore have a special role. They do not merely transform features; they decide which past positions the upper network will treat as relevant. This makes upper-layer Q/K timing an architecture--optimization interface: architecture shapes the residual directions available to attention, while optimization determines when upper attention becomes confident on them.

We study a failure mode in GPT-style decoder blocks that we call \textit{\textbf{premature upper-layer attention specialization}}: upper-layer attention becomes sharp and causally relied upon before lower-layer routing and copy-like features have stabilized. Early in pretraining, lower layers are still building routing and copy-like features associated with induction behavior \citep{olsson2022context,edelman2024evolution}. If upper-layer \qk{} logits grow too quickly, upper attention can become low-entropy on a moving, immature representation. This sharp matching is not harmless. Once a row of attention assigns most probability mass to one key, the softmax gradient available to raise alternative keys is proportional to their small probabilities. A wrong or noisy early match can therefore become an optimization attractor rather than a transient mistake.

\begin{figure*}[t]
\centering
\includegraphics[width=\textwidth]{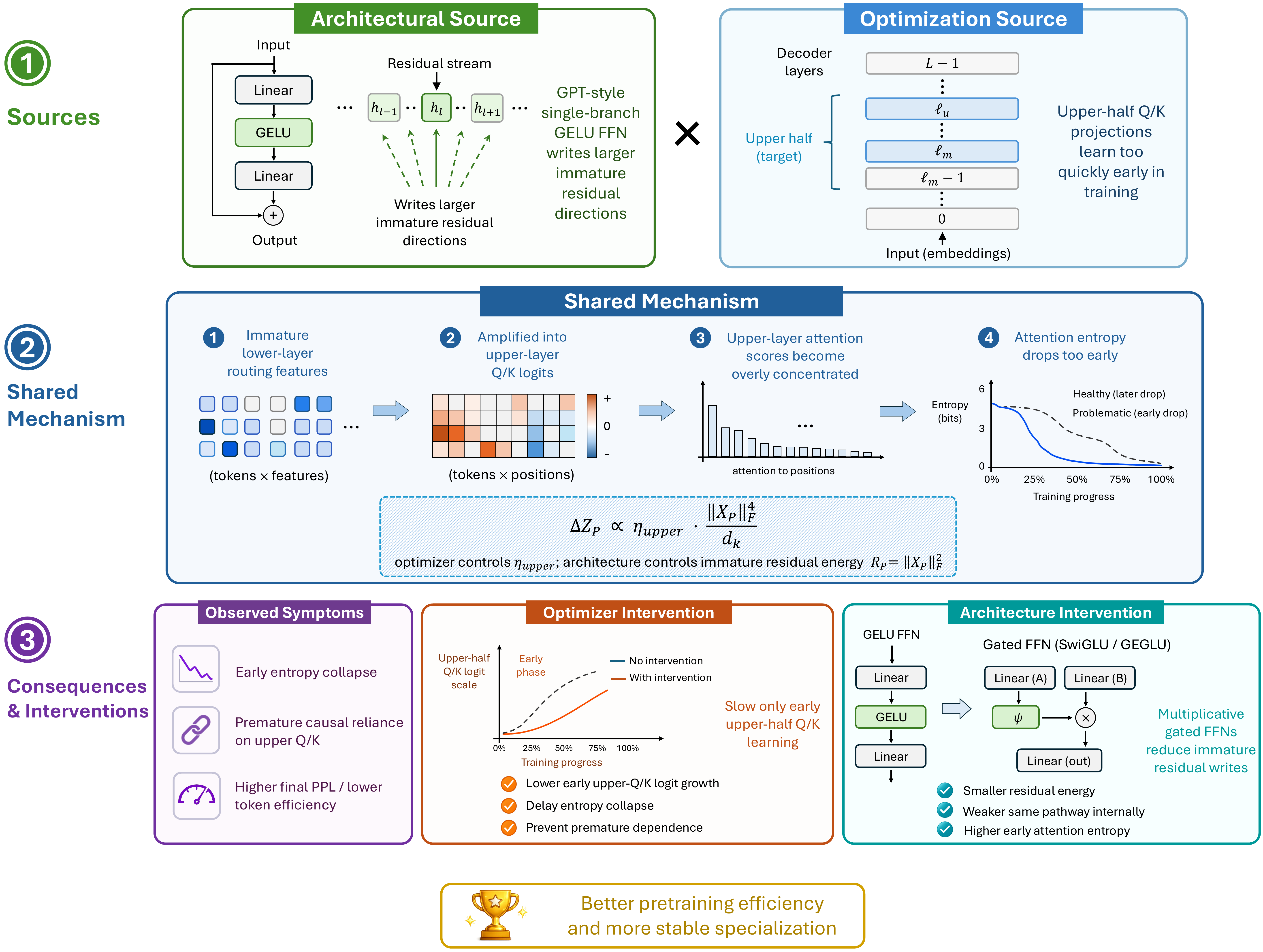}
\caption{Mechanism overview. Premature upper-layer attention specialization arises from a shared upper-\qk{} logit-growth pathway. Single-branch FFNs increase the immature residual-energy side of the pathway, while full-speed early upper-\qk{} optimization increases the step-size side. Their interaction produces early upper-logit concentration, entropy collapse, and premature causal reliance on upper \qk{}. The two interventions act on the same pathway from different sides: selective upper-\qk{} slowing reduces the optimization factor, and multiplicative gated FFNs reduce the upstream residual-write factor.}
\label{fig:mechanism-overview}
\end{figure*}

The first empirical anchor is a \textbf{selective learning intervention}. In a GPT-style 270M-parameter decoder, reducing the early learning rate only for upper-half \qk{} substantially improves final validation loss, perplexity, token efficiency, and the average score on a downstream evaluation suite. The intervention leaves lower \qk{}, values, FFNs, embeddings, and all other parameter groups on the normal optimizer schedule. Its target is therefore narrow: \textit{prevent upper attention from becoming confident before the lower residual basis has matured.}

The \textbf{mechanism evidence} shows that the gain is not a generic small-learning-rate effect. The intervention does not delay lower routing maturity. Instead, it delays upper entropy collapse, suppresses early upper \qk{} logit growth, reduces early concentration of the upper-\qk{} bilinear form, and dramatically reduces early causal dependence on upper \qk{}. The paper's central mechanism is:
\textit{\textbf{Premature upper-layer attention specialization arises when immature residual directions drive upper \qk{} logits before the lower residual basis has stabilized. A selective upper-\qk{} learning intervention suppresses the step-size side of this pathway, while multiplicative gated FFNs suppress the immature-residual-energy side.}}

Figure~\ref{fig:mechanism-overview} summarizes the mechanism studied in this paper: immature residual directions and fast early upper-\qk{} learning jointly drive premature upper attention specialization, while selective upper-\qk{} slowing and multiplicative gated FFNs suppress the same pathway from different sides.

We then ask why the intervention is much less useful in matched LLaMA-style blocks. Component ablations identify the largest suppressor: the FFN changes from a single GELU branch to a multiplicative gated branch. Matched SwiGLU and GEGLU FFNs sharply lower the marginal value of the \probe{}, separating the mechanism from parameter count and the specific SiLU activation. Thus the gated-FFN result is an architectural explanation for the upper-\qk{} intervention result.

Our contributions are summarized as follows:
% \begin{enumerate}[leftmargin=1.35em,itemsep=2pt]
\begin{itemize}
    \item We identify premature upper-layer attention specialization as a failure mode in GPT-style decoder pretraining and establish it with a targeted GPT-style intervention. The failure is localized to upper \qk{} timing rather than to attention everywhere: in the 270M GPT-style decoder, reducing only early upper-half \qk{} learning improves three-run average perplexity by 0.50 and saves 13.17\% of the training tokens needed to reach the matched baseline final loss.

    \item We provide \textbf{mechanism evidence} that the intervention leaves lower routing maturity intact while delaying upper entropy collapse, suppressing upper-logit growth, reducing concentration of the upper-\qk{} bilinear form, and reducing early upper-\qk{} causal dependence.

    \item We show why the phenomenon is architecture-dependent. RMSNorm alone and bias removal alone preserve most of the intervention gain, while matched SwiGLU and GEGLU FFNs largely eliminate it. Direct pathway measurements show that gated FFNs reduce the early FFN residual-write amplitude feeding upper attention.

    \item We prove a \textbf{decoder-block theorem} that unifies the intervention and the architecture result. The intervention reduces the upper-\qk{} step-size factor directly. Gated FFNs reduce the immature residual-energy factor entering the same bound. Under a measurable immature-channel locality condition, this yields the localized growth term $O(\eta_{\mathrm{upper}}\|X_P\|_F^4/d_k)$.
% \end{enumerate}
\end{itemize}

%% file: sections/03_premature_upper_attention_specialization.tex
\section{Premature Upper Attention Specialization}
\label{sec:premature-specialization}

For an upper decoder layer, let $X\in\mathbb{R}^{n\times d}$ be the layer-normalized hidden matrix for a sequence of length $n$. For one head,
\begin{equation}
Q=XW_Q,\qquad K=XW_K,\qquad Z=\frac{QK^\top}{\sqrt{d_k}}
  =\frac{XW_QW_K^\top X^\top}{\sqrt{d_k}},
\end{equation}
where $Z$ is the pre-softmax attention-logit matrix. Define $B=W_QW_K^\top$. Upper \qk{} specialization is the growth of a structured bilinear form $XBX^\top$ that makes upper attention selective over positions.

Specialization is useful when the residual basis $X$ already contains stable routing features. It is harmful when a component of $X$ is immature. We refer to these moving components as immature residual directions. Formally, let $P$ be an orthogonal projector onto residual directions whose lower-layer routing features have not yet stabilized, and let $X_P=XP$. If upper \qk{} learns quickly on $X_P$, the model can form confident attention matches on features that later move or disappear. We measure this event with upper attention logit magnitude, upper attention entropy, upper/lower logit ratio, and causal dependence on upper \qk{}.

We use one selective intervention to expose and correct the failure. During the early window, the intervention multiplies the learning rate of upper-half $W_Q$ and $W_K$ by $0.25$ while leaving lower \qk{}, values, FFNs, embeddings, and all other optimizer behavior unchanged. The release rule is fixed across experiments: the lower-copy score, defined as the lower-half attention mass assigned to the nearest previous occurrence of the same token and averaged over layers, heads, and valid repeated-token positions, must be at least $0.005$ for three consecutive evaluations. The earliest release is at 3\% of training; if the condition has not fired by 12\%, release is forced. After release, the multiplier ramps from $0.25$ to $1.0$ over 1\% of total steps. Appendix~\ref{app:release-rule-robustness} reports release-rule robustness checks for this maturity-based criterion. This asks a causal question with a practical payoff: does reducing only early upper \qk{} learning prevent the model from taking a worse optimization path?

%% file: sections/04_architecture_gated_ffns.tex
\section{Architecture: Why Gated FFNs Matter}

The single-branch GPT-style FFN is
\begin{equation}
F_{\mathrm{single}}(x)=W_{\mathrm{out}}\phi(W_{\mathrm{in}}x),
\end{equation}
where $\phi$ is GELU in our GPT-style baseline. The multiplicative gated FFN used by SwiGLU and GEGLU is
\begin{equation}
F_{\mathrm{gate}}(x)=W_{\mathrm{out}}\left[\psi(W_{\mathrm{gate}}x)\od W_{\mathrm{up}}x\right].
\end{equation}
The difference is structural. A single branch can pass an immature feature through one projection and one nonlinearity. A gated FFN requires two projected signals to co-occur multiplicatively. Before a feature direction has stable alignment in both branches, the product attenuates it. This creates an architectural delay: unstable residual directions are less able to drive the next upper attention block into sharp matching.

To separate multiplicative gating from the particular gate activation, we use GEGLU with the same width and the same three-matrix gated structure as SwiGLU. GEGLU differs mainly in the gate activation and suppresses the intervention gain at least as strongly as SwiGLU. The mechanism is the multiplicative gate, not the particular activation function.

%% file: sections/05_experimental_setup.tex
\section{Experimental Setup}

\paragraph{270M decoder pretraining.}
The main experiments use a 20-layer causal decoder with width 960, 15 attention heads, sequence length 1024, RoPE position encoding, tied input/output embeddings, and roughly 270M parameters. All pretraining runs use 2.5B FineWeb-Edu tokens \citep{penedo2024fineweb} and the same GPT-2 tokenizer, packed-token data pipeline, AdamW optimizer family, cosine schedule, batch geometry, and evaluation protocol. The GPT-style baseline uses LayerNorm, biased linear projections, and a GELU FFN of hidden width 3840.

\paragraph{Architectural variants.}
We evaluate a full LLaMA-style variant at the same depth, width, head count, sequence length, and token budget. It uses RMSNorm \citep{zhang2019rmsnorm}, biasless linear projections, and a SwiGLU FFN with hidden width 2560, matching the FFN parameter/FLOP budget of the GELU width-3840 block. We then isolate components with same-size paired runs: GPT+RMSNorm, GPT+biasless projections, GPT+matched SwiGLU FFN, LLaMA-style+LayerNorm, and matched GEGLU FFN.

\paragraph{Intervention comparisons.}
For each architecture, we compare a matched control with the \probe{}. The intervention benefit is reported as PPL reduction, defined as matched-control perplexity minus intervention perplexity. Larger positive values mean larger benefit from reducing early upper-\qk{} learning.

\paragraph{Mechanism measurements.}
We record upper attention logit magnitude, upper attention entropy, upper/lower logit ratio, lower copy-routing maturity, and causal sensitivity to zeroing upper \qk{}. The main causal intervention zeros upper-half $Q$, $K$, or both at evaluation time. In dot-product attention these interventions make upper attention logits degenerate in the same way, and therefore give the same upper-wide causal readout. For the gated-FFN attribution, we also measure the FFN residual write added by each block, because this is the residual-energy input to the upper-\qk{} logit-growth bound.

\begin{figure}[t]
\centering
\includegraphics[width=0.86\linewidth]{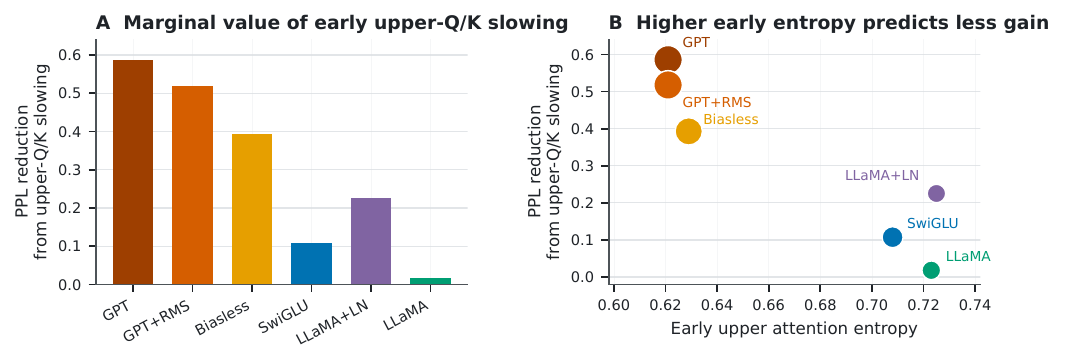}
\caption{Architecture attribution for premature upper attention specialization. Left: the marginal value of the \probe{} is large in GPT-style blocks and small in gated-FFN blocks. Right: architectures with higher early upper attention entropy have less to gain from externally reducing early upper-\qk{} learning.}
\label{fig:architecture-attribution}
\end{figure}
\input{tables/main_270m_pretraining}

%% file: tables/main_270m_pretraining.tex
\begin{table}[t]
\centering
\scriptsize
\setlength{\tabcolsep}{3pt}
\caption{GPT-style 270M pretraining. The intervention slows only upper-half \qk{} in the early window. Results are mean $\pm$ sample standard deviation across three full training runs; differences use paired seeds.}
\label{tab:main-b}
\begin{tabular}{@{}lcccc@{}}
\toprule
Model & Loss & PPL & Tokens to baseline final loss & Token saving \\
\midrule
GPT-style control & $3.2883 \pm 0.0021$ & $26.7981 \pm 0.0560$ & 2.500B & -- \\
Intervention & $3.2696 \pm 0.0051$ & $26.3014 \pm 0.1344$ & $2.171 \pm 0.035$B & $329.4$M / $13.17 \pm 1.41\%$ \\
\midrule
Difference & $-0.0187 \pm 0.0030$ & $-0.4967 \pm 0.0789$ & $-0.329 \pm 0.035$B & -- \\
\bottomrule
\end{tabular}
\end{table}

%% file: sections/06_results.tex
\section{Results}

\subsection{A GPT-Style Block Is Highly Sensitive to Early Upper \qk{} Learning}

Table~\ref{tab:main-b} shows the three-run 270M GPT-style result. The \probe{} improves average final validation loss by 0.0187 and average perplexity by 0.4967. It also reaches the matched baseline final loss after 2.171B tokens on average, saving 329.4M tokens out of the 2.5B-token budget.

The same paired intervention also improves a larger 0.7B GPT-style decoder trained for 7.0B tokens (three seeds), improving final perplexity by 0.13 and saving 0.54B same-loss tokens on average; see Appendix~\ref{app:scale-check}.

As an additional control, halving the global learning rate worsens final perplexity to 31.42, confirming that the gain is not generic slower training but selective early upper-\qk{} slowing; see Appendix~\ref{app:global-lr-control}.

The release criterion is also stable: offline replay of nearby lower-copy thresholds and patience values releases in the same early 3--6\% window, and fixed-release controls at 3\% and 6\% preserve most of the perplexity and token-efficiency gain; see Appendix~\ref{app:release-rule-robustness}.

\subsection{Downstream Benchmark Evaluation}

Table~\ref{tab:downstream-upperlr} reports the downstream evaluation for the same 270M GPT-style intervention comparison. The \probe{} improves the three-run average score by 0.41 percentage points and wins 14 of 21 run-task comparisons. The gains are largest on LAMBADA, RACE, and ARC-Easy, while ARC-Challenge is nearly flat. Together with the validation and token-efficiency results, this shows that reducing premature upper-attention specialization improves the trained base model's general evaluation behavior.

\input{tables/downstream_upperlr}

\subsection{The Failure Is Early Causal Dependence on Upper \qk{}}

The mechanism should not be inferred from loss alone. At 3\% of training, the GPT-style control already has upper attention with high logit magnitude and low entropy. More importantly, upper \qk{} is already causally necessary. Figure~\ref{fig:mechanism}D reports the perplexity increase from zeroing upper-half \qk{} at several checkpoints. In the GPT-style control, upper \qk{} ablation increases perplexity by 82.3 at 3\% and by 146.1 at the end. Under the intervention, the same ablation increases perplexity by only 0.35 at 3\%, then grows gradually to 17.4 at the end.

This result is sharper than a statement about delayed training curves. The intervention does not merely shift the same upper-\qk{} dependence curve to the right. It prevents the model from becoming prematurely and excessively dependent on upper \qk{}. Upper \qk{} still becomes useful later, but the final model is much less brittle to upper-\qk{} ablation.

\subsection{The Intervention Delays Upper Specialization, Not Lower Routing}

Figure~\ref{fig:mechanism} separates the developmental events and connects them to the causal readout. Lower routing maturity appears at the same training progress under the control and the intervention. The difference is upper attention: in the control, upper attention reaches the sharp-entropy threshold at the same early checkpoint as lower routing maturity; under the intervention, upper sharpness is delayed to 20\% of training. The same pattern appears in the raw readouts. At 3\% of training, the intervention has much higher upper attention entropy, much lower upper logit magnitude, and much smaller top singular value of the upper-\qk{} bilinear matrix $W_Q^\top W_K$.

A diagnostic sweep over the early upper-\qk{} multiplier gives the same direction: smaller multipliers monotonically reduce early upper-logit growth and upper-\qk{} bilinear displacement; see Appendix~\ref{app:alpha-sweep}.

\subsection{The Phenomenon Is Architecture-Dependent}

The full LLaMA-style 270M block barely benefits from the \probe{}: final perplexity changes from 26.3712 to 26.3531, a gain of only 0.0182. In the matched controlled comparison used for the component sweep, the GPT-style reference gains 0.5858 perplexity. The question is which component explains the difference.

Figure~\ref{fig:architecture-attribution} summarizes the architecture-level pattern: architectures with lower early upper attention entropy obtain larger benefits from early upper-\qk{} slowing. Table~\ref{tab:components} isolates the components. RMSNorm alone does not fix the failure: GPT+RMSNorm keeps the same early upper logit and entropy as the GPT-style reference and retains a large intervention gain. Removing biases helps but still leaves a substantial gain. The largest suppressor is the FFN change. A matched SwiGLU FFN reduces the intervention gain to 0.1075 perplexity while also lowering early upper logit magnitude and raising upper entropy. LLaMA-style+LayerNorm sits between the two, showing that the full LLaMA-style behavior is an interaction, with RMSNorm helping once the gated FFN and biasless projections are present.

\input{tables/component_attribution}

\begin{figure*}[t]
\centering
\includegraphics[width=\textwidth]{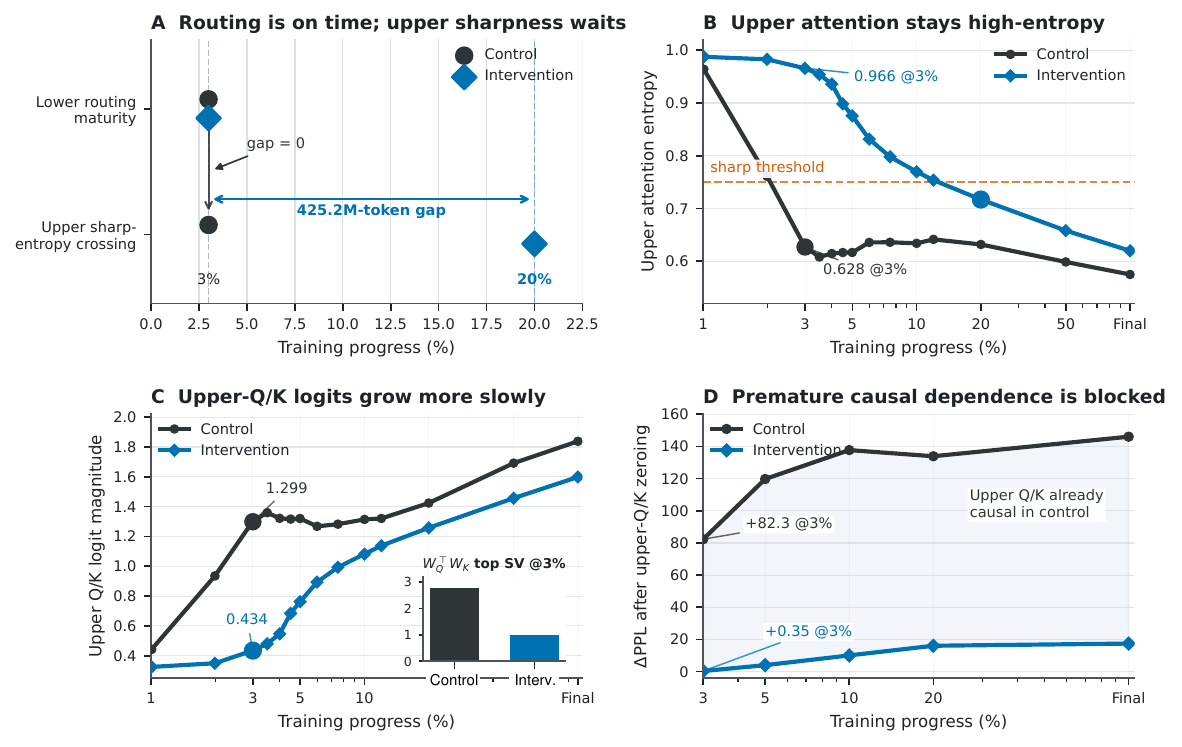}
\caption{The intervention delays upper attention specialization without delaying lower routing. Control denotes default training; Intervention denotes reducing only early upper-half \qk{} learning by a factor of 0.25 before release. (A) Lower routing matures at 3\% in both runs, while upper sharp-entropy crossing moves from 3\% in the control to 20\% under the intervention. (B) Upper attention entropy stays high longer under the intervention. (C) Upper-\qk{} logit growth is suppressed, with the 3\% upper $W_Q^\top W_K$ top singular value shown as an inset. (D) Zeroing upper-half \qk{} reveals that the control becomes prematurely dependent on upper \qk{}, whereas the intervention sharply reduces that dependence.}
\label{fig:mechanism}
\end{figure*}

\subsection{The Suppressor Is the Multiplicative Gate}

To isolate the multiplicative structure itself, we compare SwiGLU and GEGLU with the same gated FFN width, the same three linear maps, and the same parameter count. Table~\ref{tab:gated} shows that both gated FFNs sharply reduce the marginal value of the \probe{}. The GEGLU control has final perplexity 26.3627, and adding the intervention improves it by only 0.0549. Its early upper attention is also less sharp than the GPT-style control: upper logit 0.983 and entropy 0.729 at 3\%.

\input{tables/gated_ffn_attribution}

Direct pathway measurements close the architecture attribution. Table~\ref{tab:ffn-pathway} measures the FFN residual write at the same 3\% checkpoint used for the attention diagnostics. Matched SwiGLU and GEGLU reduce upper-layer FFN residual-write RMS by 55.0\% and 52.4\% relative to the GELU-3840 GPT-style FFN, and by 42.2\% and 38.9\% relative to a width-matched GELU-2560 capacity control. The same checkpoints have lower upper \qk{} logit magnitude and higher upper attention entropy. This identifies the gated FFN effect with the residual-energy term in Section~\ref{sec:theory}: gated FFNs reduce the residual input that can drive premature upper-\qk{} logit growth.

\input{tables/ffn_pathway_evidence}

Appendix~\ref{app:entropy-floor-control} adds a mediator control: directly imposing an upper-attention entropy floor raises early entropy and gives a small perplexity improvement, but it does not recover the much larger gain from changing the upper-\qk{} timing path.

%% file: tables/downstream_upperlr.tex
\begin{table}[t]
\centering
\scriptsize
\caption{Downstream benchmark evaluation for the GPT-style intervention comparison. Results are mean $\pm$ sample standard deviation over three runs on the pre-specified seven-task suite. Delta is the paired intervention-minus-control difference in percentage points; the mean-row deviation is computed over seed-level suite averages.}
\label{tab:downstream-upperlr}
\begin{tabular}{lccc}
\toprule
Task & Control & Intervention & Delta pp \\
\midrule
PIQA & $0.6065 \pm 0.0056$ & $0.6103 \pm 0.0013$ & $+0.381 \pm 0.544$ \\
HellaSwag & $0.3011 \pm 0.0013$ & $0.3024 \pm 0.0020$ & $+0.133 \pm 0.314$ \\
ARC-Easy & $0.4315 \pm 0.0125$ & $0.4373 \pm 0.0034$ & $+0.575 \pm 1.582$ \\
ARC-Challenge & $0.2466 \pm 0.0023$ & $0.2460 \pm 0.0125$ & $-0.057 \pm 1.258$ \\
RACE & $0.2695 \pm 0.0049$ & $0.2778 \pm 0.0056$ & $+0.829 \pm 1.023$ \\
LAMBADA & $0.1974 \pm 0.0045$ & $0.2073 \pm 0.0004$ & $+0.996 \pm 0.485$ \\
MMLU & $0.2291 \pm 0.0009$ & $0.2294 \pm 0.0001$ & $+0.028 \pm 0.082$ \\
\midrule
Mean & -- & -- & $+0.412 \pm 0.143$ \\
\bottomrule
\end{tabular}
\end{table}

%% file: tables/component_attribution.tex
\begin{table}[t]
\centering
\footnotesize
\caption{Same-size component attribution in the 270M decoder, using one matched controlled comparison for each architecture pair. PPL reduction is matched-control perplexity minus intervention perplexity, so larger values mean larger benefit from reducing early upper-\qk{} learning. Early readouts are measured at about 3\% of training.}
\label{tab:components}
\begin{tabular}{lrrr}
\toprule
Architecture & PPL reduction & Logit @3\% & Entropy @3\% \\
\midrule
GPT-style reference & 0.5858 & 1.308 & 0.621 \\
GPT + RMSNorm & 0.5179 & 1.308 & 0.621 \\
GPT + biasless projections & 0.3928 & 1.263 & 0.629 \\
GPT + matched SwiGLU FFN & 0.1075 & 1.044 & 0.708 \\
LLaMA-style + LayerNorm & 0.2253 & 0.972 & 0.725 \\
LLaMA-style reference & 0.0182 & 0.979 & 0.723 \\
\bottomrule
\end{tabular}
\end{table}

%% file: tables/gated_ffn_attribution.tex
\begin{table}[t]
\centering
\small
\setlength{\tabcolsep}{4pt}
\caption{Clean gated-FFN attribution. SwiGLU and GEGLU use matched gated width and the same parameter count. Both suppress the PPL reduction from the intervention, separating the mechanism from the SiLU activation.}
\label{tab:gated}
\begin{tabular}{lrrrr}
\toprule
Architecture & Control ppl & Intervention ppl & PPL reduction & Upper entropy @3\% \\
\midrule
GPT-style GELU FFN & 26.7382 & 26.1524 & 0.5858 & 0.621 \\
SwiGLU gated FFN & 26.2083 & 26.1008 & 0.1075 & 0.708 \\
GEGLU gated FFN & 26.3627 & 26.3078 & 0.0549 & 0.729 \\
\bottomrule
\end{tabular}
\end{table}

%% file: tables/ffn_pathway_evidence.tex
\begin{table}[ht]
\centering
\caption{Direct FFN pathway evidence at 3\% training. Gated FFNs reduce the early FFN residual-write amplitude feeding upper attention while simultaneously reducing upper \qk{} logit magnitude and increasing upper attention entropy.}
\label{tab:ffn-pathway}
\small
\begin{tabular}{lrrrr}
\toprule
Architecture & Upper FFN write RMS & Upper logit & Upper entropy & First-upper write RMS \\
\midrule
GELU-3840 & 0.1205 & 1.3128 & 0.6277 & 0.1262 \\
GELU-2560 & 0.0938 & 1.1496 & 0.6750 & 0.0960 \\
SwiGLU-2560 & 0.0542 & 0.9491 & 0.7423 & 0.0469 \\
GEGLU-2560 & 0.0573 & 0.9994 & 0.7226 & 0.0505 \\
\bottomrule
\end{tabular}
\end{table}

%% file: sections/07_theory_summary.tex
\section{Theory}
\label{sec:theory}

We give a decoder-block certificate for the pathway in Figure~\ref{fig:mechanism-overview}. Let $X\in\mathbb{R}^{n\times d}$ be the normalized residual stream entering an upper attention head, $B=W_QW_K^\top$, and $Z=XBX^\top/\sqrt{d_k}$. Let $P=P^\top=P^2$ project onto immature residual directions and write $X_P=XP$. The immature logit component is
\begin{equation}
Z_P=\frac{X_P PBP X_P^\top}{\sqrt{d_k}} .
\end{equation}

\begin{theorem}[Pathwise premature-logit growth]
\label{thm:pathwise}
For one effective gradient step on $W_Q,W_K$ with upper step sizes $\eta_Q,\eta_K$, the first-order immature-logit update obeys a deterministic pathwise bound
\begin{equation}
\begin{aligned}
\|\Delta Z_P\|_F
\le
\frac{\|X_P\|_{\mathrm{op}}^2\|X_P\|_F\|E\|_{\mathrm{op}}}{d_k}
\Big(
\eta_Q\|XW_KW_K^\top P\|_F+
\eta_K\|XW_QW_Q^\top P\|_F
\Big)
+O(\eta_Q\eta_K),
\end{aligned}
\end{equation}
where $E=\partial\mathcal{L}/\partial Z$ is the masked-softmax adjoint. Without extra assumptions this gives an unconditional bound controlled by the total residual scale. Under the measurable immature-channel locality condition
\begin{equation}
\|XW_SW_S^\top P\|_F
\le
\lambda_S\|X_P\|_F\|W_S\|_{\mathrm{op}}^2,\qquad S\in\{Q,K\},
\end{equation}
it sharpens to
\begin{equation}
\|\Delta Z_P\|_F
=
O\!\left(
\eta_{\mathrm{upper}}\frac{\|X_P\|_F^4}{d_k}
\right).
\end{equation}
\end{theorem}

Appendix~\ref{app:locality-diagnostic} reports the corresponding early-checkpoint locality ratios; both $\lambda_Q$ and $\lambda_K$ remain below one at the 3\% mechanism checkpoint for the control and intervention runs.

Theorem~\ref{thm:pathwise} proves the optimizer side of the mechanism. The intervention sets $\eta_Q=\eta_K=\alpha\eta$ for early upper-half \qk{}, with $\alpha=0.25$ in the main experiments, so $\Delta Z_P^{(\alpha)}=\alpha\Delta Z_P^{(1)}+O(\eta^2)$ at a fixed checkpoint. This is the formal reason that reducing only early upper-\qk{} learning suppresses premature upper-logit growth while leaving lower layers, values, FFNs, embeddings, and the output head on their normal schedule. The dose-response experiment in Appendix~\ref{app:alpha-sweep} matches the predicted monotone suppression of early upper-logit growth.

The architecture side follows from the same bound. A single-branch FFN has $F_{\mathrm{single}}(x)=W_o\phi(W_ix)$, whereas a gated FFN has
\begin{equation}
F_{\mathrm{gate}}(x)=\widetilde W_o\left[\psi(W_gx)\odot W_ux\right].
\end{equation}
Under matched initialization, independent branches, and a residual-output projector $P_{\mathrm{out}}$, the gated output energy satisfies
\begin{equation}
\frac{
\mathbb{E}\|P_{\mathrm{out}}F_{\mathrm{gate}}(x)\|_2^2
}{
\mathbb{E}\|P_{\mathrm{out}}F_{\mathrm{single}}(x)\|_2^2
}
=
\rho_0
=
\frac{\tau_g^2}{\tau_s^2}
\frac{r}{m}
\frac{\nu\mathbb{E}\psi(G_\nu)^2}{\mathbb{E}\phi(G_\nu)^2}.
\end{equation}
For the matched 270M FFN widths, $m=3840$ and $r=2560$, giving $\rho_{\mathrm{GEGLU}}\approx0.256$ and $\rho_{\mathrm{SwiGLU}}\approx0.216$ at initialization. With the residual connection, this gives a strict contraction of the FFN-added immature residual energy whenever the single-branch FFN contributes nonzero energy to that subspace.

Let $R_P=\|X_P\|_F^2$ be the immature residual energy entering upper attention. Combining the localized upper-\qk{} bound with the gated-FFN contraction gives
\begin{equation}
\|\Delta Z_P\|_F
\le
C_t\eta_{\mathrm{upper}}R_P^2+O(\eta^2).
\end{equation}
If
$
R_P^{\mathrm{gate}}\le\bar\rho R_P^{\mathrm{single}}
$
at the same upper attention input, then
\begin{equation}
\|\Delta Z_P^{\mathrm{gate}}\|_F
\le
\bar\rho^2\|\Delta Z_P^{\mathrm{single}}\|_F+O(\eta^2).
\end{equation}
Thus the intervention and the gated FFN suppress the same pathway from different sides: the intervention reduces $\eta_{\mathrm{upper}}$, while the gated FFN reduces $R_P$ in the shared term $\eta_{\mathrm{upper}}R_P^2=\eta_{\mathrm{upper}}\|X_P\|_F^4$. Finally, low-entropy attention requires logit range: if one of $N$ causal keys receives probability at least $1-\epsilon$, then
\begin{equation}
\max_j z_j-\min_j z_j
\ge
\log\frac{(N-1)(1-\epsilon)}{\epsilon}.
\end{equation}
Suppressing immature-logit growth therefore delays premature upper attention specialization. Appendix~\ref{app:full-theory} gives the complete proofs.

%% file: sections/08_discussion.tex
\section{Discussion}

The central finding is premature upper-layer attention specialization: in GPT-style decoder pretraining, upper \qk{} can become sharp and causally necessary before the residual basis is stable enough to support that specialization. Selective early reduction of upper-\qk{} learning establishes the failure experimentally: lower routing develops on schedule, while upper attention avoids premature low-entropy specialization and excessive early causal dependence.

The architecture attribution explains when the intervention is needed. Single-branch FFNs let immature residual directions feed upper \qk{} matching more strongly. Gated FFNs regularize the same path internally: the multiplicative product injects less immature FFN energy into the residual stream before those directions can dominate upper matching. Direct pathway measurements show this contraction at the FFN-write level, and the theory proves the corresponding pathwise decoder-block result. The learning-rate intervention reduces the upper-\qk{} step-size factor; the gated FFN reduces the immature residual-energy factor entering the same bound. Under the measurable locality condition, these combine into the localized growth term $O(\eta_{\mathrm{upper}}\|X_P\|_F^4/d_k)$.

The result should not be reduced to ``use SwiGLU''. SwiGLU and GEGLU both work in the relevant sense, and they share multiplicative gating. Nor is the mechanism only a global attention-entropy story: the entropy-floor control in Appendix~\ref{app:entropy-floor-control} shows that directly raising entropy gives only a small gain, while gated FFNs identify the upstream residual-write mechanism because they act before the upper attention logits are formed.

%% file: sections/09_limitations.tex
% \section{Limitations}

%% file: sections/10_conclusion.tex
\section{Conclusion and Limitation}

Premature upper-layer attention specialization is an optimization failure in GPT-style decoder pretraining. The targeted upper-\qk{} learning-rate intervention establishes the failure and its practical cost: the baseline becomes sharply and causally dependent on upper \qk{} very early, while reducing early upper-\qk{} learning improves final perplexity, token efficiency, and downstream average score. The mechanism is delayed and weakened premature upper-\qk{} dependence, not delayed lower routing. The architecture attribution then explains why this issue is smaller in gated blocks: matched multiplicative gated FFNs, including both SwiGLU and GEGLU, contract immature residual directions before they drive upper \qk{} logits. Learning less in early upper attention is therefore not merely an optimizer recipe; it reveals a concrete timing failure in decoder pretraining and a structural way to suppress it.
\textit{\textbf{Limitation.}}
The experiments are conducted at academic pretraining scale rather than frontier-industrial scale.
\textit{Additionally,}
related work is provided in Appendix \ref{sec:Related-Work}.

%% file: sections/02_related_work.tex
\section{Related Work}
\label{sec:Related-Work}

\paragraph{Decoder pretraining and training-time analysis.}
Causal decoder pretraining is the standard setting for studying language-model scaling, from the Transformer and GPT-style decoders to GPT-3, Megatron-LM, PaLM, OPT, BLOOM, LLaMA, and LLaMA~2 \citep{vaswani2017attention,radford2019language,brown2020language,shoeybi2019megatron,chowdhery2023palm,zhang2022opt,bigscience2022bloom,touvron2023llama,touvron2023llama2}. Scaling laws characterize loss as a function of model size, data, and compute \citep{kaplan2020scaling,hoffmann2022training}; Pythia, OLMo, and PolyPythias make training-time development observable through open checkpoints, data, code, and repeated runs \citep{biderman2023pythia,groeneveld2024olmo,vanderwal2025polypythias}. We take this developmental view inside the decoder block: upper attention can become confident before lower copy and routing features are ready.

\paragraph{Transformer architecture, normalization, and stability.}
Transformer stability depends on block-level choices such as normalization placement, residual scaling, positional encoding, and FFN parameterization. LayerNorm and RMSNorm are standard \citep{ba2016layer,zhang2019rmsnorm}; PreNorm, ScaleNorm, NormFormer, DeepNorm, and related analyses study how normalization and residual parameterization affect depth and optimization stability \citep{nguyen2019transformers,xiong2020layernorm,shleifer2021normformer,wang2024deepnet,liu2020understanding}. RoPE is now common in decoder LMs \citep{su2024roformer}, while broad component studies show that Transformer modifications do not transfer uniformly across implementations \citep{narang2021transformer}. Our attribution is more specific: RMSNorm and bias removal help only modestly, whereas multiplicative gated FFNs strongly suppress premature upper attention specialization.

\paragraph{Feed-forward networks, gated FFNs, and residual writes.}
FFNs are not merely capacity reservoirs. Mechanistic work shows that they act as key-value memories and write interpretable residual updates in vocabulary space \citep{geva2021transformer,geva2022transformer,dai2022knowledge}. GLU-style gating was introduced for sequence models and adapted to Transformer FFNs; GEGLU and SwiGLU replace a single activation branch with a product of two projected signals \citep{dauphin2017language,shazeer2020glu,hendrycks2016gelu,ramachandran2017swish}. Modern decoder LMs such as PaLM and LLaMA use SwiGLU-style FFNs \citep{chowdhery2023palm,touvron2023llama}. We connect these lines: gated FFNs improve pretraining partly because they reduce immature FFN residual-write energy before it can drive upper-\qk{} logit growth.

\paragraph{Mechanistic development of attention and routing.}
Transformer-circuits work studies attention heads and residual streams as compositional circuits \citep{elhage2021mathematical}. Induction heads emerge during training and support copy-like routing behavior \citep{olsson2022context,edelman2024evolution}; pruning, probing, and causal interventions identify which attention heads matter for model behavior \citep{michel2019sixteen,voita2019analyzing,clark2019bert,wang2022interpretability}. This motivates our lower-routing readout and our causal upper-\qk{} ablation. The new finding is timing: upper \qk{} can become sharp and causally relied upon before lower routing has stabilized.

\paragraph{Attention entropy, Q/K logits, and closest interventions.}
Sharp attention is useful when it is mature; the failure here is premature sharpness. Prior work links low attention entropy and large logits to Transformer instability: entropy collapse motivates spectral reparameterization, QKNorm normalizes query/key activations, and ViT-22B uses QK normalization as part of a stable scaling recipe \citep{zhai2023stabilizing,henry2020query,dehghani2023scaling}. Most closely, attention-logit change can be controlled with query/key-specific learning rates \citep{anson2025controlling}. Our intervention is temporally and spatially narrower: reduce only early upper-half \qk{} learning, preserve lower-layer learning, and show that gated FFNs suppress the same pathway from the residual-energy side.

% All references above have been checked and verified.

%% file: sections/07_theory.tex
\section{Full Theoretical Derivations}
\label{app:full-theory}

We now prove a pathwise decoder-block theorem that unifies the two empirical levers studied above. The upper-\qk{} learning-rate intervention reduces the effective step size of the upper attention matching matrices. The gated-FFN architecture reduces the immature residual energy entering the same matching pathway. The proof is stated directly in a pre-norm decoder block, using the residual stream, the \qk{} bilinear form, the masked-softmax adjoint, and the FFN residual update.

Let $X\in\mathbb{R}^{n\times d}$ be the normalized residual stream entering one upper-layer attention head, and let
\begin{equation}
W_Q,W_K\in\mathbb{R}^{d\times d_k},\qquad B=W_QW_K^\top .
\end{equation}
The pre-softmax attention logits are
\begin{equation}
Z=\frac{XBX^\top}{\sqrt{d_k}}.
\end{equation}
Let $P=P^\top=P^2$ be an orthogonal projector onto immature residual directions and define $X_P=XP$. We isolate the immature-to-immature logit component
\begin{equation}
Z_P
=
\frac{X_P\,PBP\,X_P^\top}{\sqrt{d_k}}.
\end{equation}
Let $E=\partial \mathcal{L}/\partial Z$ be the adjoint through the masked rowwise softmax. The gradients of the head logits with respect to $W_Q,W_K$ are
\begin{equation}
G_Q
=
\frac{X^\top E X W_K}{\sqrt{d_k}},
\qquad
G_K
=
\frac{X^\top E^\top X W_Q}{\sqrt{d_k}}.
\end{equation}

\subsection{Proof of Theorem~\ref{thm:pathwise}}

For completeness, we restate Theorem~\ref{thm:pathwise} with the explicit second-order remainder.

\paragraph{Full statement of Theorem~\ref{thm:pathwise}.}
Suppose one effective gradient step updates
\begin{equation}
W_Q^+ = W_Q-\eta_QG_Q,\qquad
W_K^+ = W_K-\eta_KG_K .
\end{equation}
Let $B^+=W_Q^+(W_K^+)^\top$, and let
\begin{equation}
\Delta Z_P
=
\frac{X_P\,P(B^+-B)P\,X_P^\top}{\sqrt{d_k}} .
\end{equation}
Then
\begin{equation}
\begin{aligned}
\|\Delta Z_P\|_F
\le
\frac{
\|X_P\|_{\mathrm{op}}^2
\|X_P\|_F
\|E\|_{\mathrm{op}}
}{d_k}
\Big(
&
\eta_Q
\|XW_KW_K^\top P\|_F
\\
+
&
\eta_K
\|XW_QW_Q^\top P\|_F
\Big)
+R_2 ,
\end{aligned}
\end{equation}
where the second-order remainder satisfies
\begin{equation}
R_2
\le
\frac{
\eta_Q\eta_K
\|X_P\|_{\mathrm{op}}^2
\|X_P\|_F^2
\|E\|_{\mathrm{op}}^2
\|XW_K\|_F
\|XW_Q\|_F
}{d_k^{3/2}} .
\end{equation}
In particular, ignoring $O(\eta_Q\eta_K)$ terms,
\begin{equation}
\|\Delta Z_P\|_F
=
O\!\left(
(\eta_Q+\eta_K)
\|E\|_{\mathrm{op}}
\|X_P\|_{\mathrm{op}}^2
\|X_P\|_F
\right),
\end{equation}
with the remaining dependence carried by how the current \qk{} quadratic form maps into the immature subspace.

\begin{proof}
The bilinear matrix update is
\begin{equation}
\begin{aligned}
B^+-B
&=
(W_Q-\eta_QG_Q)(W_K-\eta_KG_K)^\top
-
W_QW_K^\top
\\
&=
-\eta_QG_QW_K^\top
-\eta_KW_QG_K^\top
+\eta_Q\eta_KG_QG_K^\top .
\end{aligned}
\end{equation}
For the first-order $Q$-term,
\begin{equation}
PG_QW_K^\top P
=
\frac{
PX^\top E X W_KW_K^\top P
}{\sqrt{d_k}}
=
\frac{
X_P^\top E X W_KW_K^\top P
}{\sqrt{d_k}} .
\end{equation}
Therefore,
\begin{equation}
\|PG_QW_K^\top P\|_F
\le
\frac{
\|X_P\|_F
\|E\|_{\mathrm{op}}
\|XW_KW_K^\top P\|_F
}{\sqrt{d_k}} .
\end{equation}
Similarly,
\begin{equation}
PW_QG_K^\top P
=
\frac{
PW_QW_Q^\top X^\top E X_P
}{\sqrt{d_k}},
\end{equation}
and hence
\begin{equation}
\|PW_QG_K^\top P\|_F
\le
\frac{
\|XW_QW_Q^\top P\|_F
\|E\|_{\mathrm{op}}
\|X_P\|_F
}{\sqrt{d_k}} .
\end{equation}
Finally,
\begin{equation}
\|X_P\,P\Delta BP\,X_P^\top\|_F
\le
\|X_P\|_{\mathrm{op}}^2
\|P\Delta BP\|_F .
\end{equation}
Dividing by the outer $\sqrt{d_k}$ in the definition of $\Delta Z_P$ gives the first-order bound.

For the second-order term,
\begin{equation}
\|PG_QG_K^\top P\|_F
\le
\|PG_Q\|_F
\|G_K^\top P\|_F .
\end{equation}
Using
\begin{equation}
\|PG_Q\|_F
\le
\frac{
\|X_P\|_F
\|E\|_{\mathrm{op}}
\|XW_K\|_F
}{\sqrt{d_k}},
\end{equation}
and
\begin{equation}
\|G_K^\top P\|_F
\le
\frac{
\|XW_Q\|_F
\|E\|_{\mathrm{op}}
\|X_P\|_F
}{\sqrt{d_k}},
\end{equation}
then multiplying again by the outer factor $\|X_P\|_{\mathrm{op}}^2/\sqrt{d_k}$ gives the claimed $R_2$.
\end{proof}

\begin{theorembox}
\begin{corollary}[Unconditional pathwise bound]
Without any additional structural assumption,
\begin{equation}
\|XW_SW_S^\top P\|_F
\le
\|X\|_F\|W_S\|_{\mathrm{op}}^2,
\qquad S\in\{Q,K\}.
\end{equation}
Therefore,
\begin{equation}
\|\Delta Z_P\|_F
\le
\frac{
\|E\|_{\mathrm{op}}
\|X\|_F
}{d_k}
\Big(
\eta_Q\|W_K\|_{\mathrm{op}}^2
+
\eta_K\|W_Q\|_{\mathrm{op}}^2
\Big)
\|X_P\|_{\mathrm{op}}^2\|X_P\|_F
+
R_2 .
\end{equation}
Using $\|X_P\|_{\mathrm{op}}\le \|X_P\|_F$ gives a cubic dependence on immature residual amplitude, modulated by the total residual scale. Since pre-norm keeps the total sequence-scale residual norm controlled, this establishes the first pillar of the mechanism: reducing either the upper-\qk{} step size or the immature residual energy reduces the fastest possible premature-logit growth.
\end{corollary}
\end{theorembox}

\begin{theorembox}
\begin{corollary}[Localized immature-channel bound]
Assume the following locality condition holds for the immature subspace $P$:
\begin{equation}
\|XW_SW_S^\top P\|_F
\le
\lambda_S
\|X_P\|_F
\|W_S\|_{\mathrm{op}}^2,
\qquad S\in\{Q,K\}.
\end{equation}
This says that the part of the upper \qk{} quadratic form that lands in immature directions is not mainly driven by mature residual directions. Under this condition,
\begin{equation}
\|\Delta Z_P\|_F
\le
\frac{
\|X_P\|_{\mathrm{op}}^2
\|X_P\|_F^2
\|E\|_{\mathrm{op}}
}{d_k}
\Big(
\eta_Q\lambda_K\|W_K\|_{\mathrm{op}}^2
+
\eta_K\lambda_Q\|W_Q\|_{\mathrm{op}}^2
\Big)
+
R_2 .
\end{equation}
Using $\|X_P\|_{\mathrm{op}}\le \|X_P\|_F$, this gives
\begin{equation}
\|\Delta Z_P\|_F
=
O\!\left(
\eta_{\mathrm{upper}}
\frac{
\|X_P\|_F^4
}{d_k}
\right).
\end{equation}
This is the rigorous version of the localized $\eta_{\mathrm{upper}}\|X_P\|_F^4$ statement. Appendix~\ref{app:locality-diagnostic} gives the corresponding diagnostic ratio.
\end{corollary}
\end{theorembox}

\begin{theorembox}
\begin{corollary}[Upper-\qk{} learning-rate suppression]
Suppose $Q$ and $K$ use the same effective upper learning rate $\eta$. Apply an early upper-\qk{} multiplier $0<\alpha<1$, so that
\begin{equation}
\eta_Q'=\eta_K'=\alpha\eta .
\end{equation}
At a fixed checkpoint with the same $X,E,W_Q,W_K$,
\begin{equation}
\Delta Z_P^{(\alpha)}
=
\alpha \Delta Z_P^{(1)}
+
O(\eta^2).
\end{equation}
Consequently,
\begin{equation}
\|\Delta Z_P^{(\alpha)}\|_F
\le
\alpha
\|\Delta Z_P^{(1)}\|_F
+
O(\eta^2).
\end{equation}
For the intervention used in this paper, $\alpha=0.25$, so the first-order immature upper-logit growth channel is reduced fourfold.
\end{corollary}
\end{theorembox}

This is the formal statement behind the selective learning-rate intervention. Lower layers and FFNs can continue to train normally, while the premature upper-\qk{} logit-growth pathway is slowed directly. Appendix~\ref{app:alpha-sweep} reports the corresponding multiplier sweep.

\subsection{Matched Gated FFNs Reduce the Immature Residual Energy Entering the Bound}

We next show why multiplicative gated FFNs reduce the same path. Consider a single-branch FFN
\begin{equation}
F_{\mathrm{single}}(x)
=
W_o\phi(W_ix),
\end{equation}
with hidden width $m$, and a gated FFN
\begin{equation}
F_{\mathrm{gate}}(x)
=
\widetilde W_o
\left[
\psi(W_gx)\odot W_ux
\right],
\end{equation}
with gated width $r$. Let $P_{\mathrm{out}}$ be an output residual subspace projector, such as the immature residual subspace entering the next upper attention layer. Assume $W_i,W_g,W_u\sim\mathcal{N}(0,\sigma^2)$, and output projections have variances $\tau_s^2$ and $\tau_g^2$. For a normalized input $x$, define $\nu=\sigma^2\|x\|_2^2$ and let $G_\nu\sim\mathcal{N}(0,\nu)$.

\begin{theorembox}
\begin{theorem}[Early gated-FFN output contraction]
At initialization,
\begin{equation}
\mathbb{E}
\|P_{\mathrm{out}}F_{\mathrm{single}}(x)\|_2^2
=
\operatorname{rank}(P_{\mathrm{out}})
\tau_s^2
m
\mathbb{E}\phi(G_\nu)^2 ,
\end{equation}
while
\begin{equation}
\mathbb{E}
\|P_{\mathrm{out}}F_{\mathrm{gate}}(x)\|_2^2
=
\operatorname{rank}(P_{\mathrm{out}})
\tau_g^2
r
\nu
\mathbb{E}\psi(G_\nu)^2 .
\end{equation}
Hence
\begin{equation}
\frac{
\mathbb{E}
\|P_{\mathrm{out}}F_{\mathrm{gate}}(x)\|_2^2
}{
\mathbb{E}
\|P_{\mathrm{out}}F_{\mathrm{single}}(x)\|_2^2
}
=
\rho_0,
\end{equation}
where
\begin{equation}
\rho_0
=
\frac{\tau_g^2}{\tau_s^2}
\frac{r}{m}
\frac{
\nu\mathbb{E}\psi(G_\nu)^2
}{
\mathbb{E}\phi(G_\nu)^2
}.
\end{equation}
For the matched 270M comparison in this paper, $m=3840$, $r=2560$, so $r/m=2/3$. With the stated initialization giving $\nu\approx0.384$, matched GEGLU gives $\rho_{\mathrm{GEGLU}}\approx0.256$, and matched SwiGLU gives approximately $\rho_{\mathrm{SwiGLU}}\approx0.216$.
\end{theorem}
\end{theorembox}

\begin{proof}
Conditional on a hidden vector $h$, an isotropic output projection satisfies
\begin{equation}
\mathbb{E}_{W_o}
\|P_{\mathrm{out}}W_oh\|_2^2
=
\operatorname{rank}(P_{\mathrm{out}})
\tau_s^2
\|h\|_2^2 .
\end{equation}
For the single-branch FFN, $h_i=\phi(a_i^\top x)$ with $a_i^\top x\sim G_\nu$. Therefore,
\begin{equation}
\mathbb{E}\|h\|_2^2
=
m\mathbb{E}\phi(G_\nu)^2 .
\end{equation}
For the gated FFN,
\begin{equation}
\tilde h_i
=
\psi(b_i^\top x)(c_i^\top x),
\end{equation}
with $b_i^\top x$ and $c_i^\top x$ independent copies of $G_\nu$. Hence
\begin{equation}
\mathbb{E}\tilde h_i^2
=
\mathbb{E}\psi(G_\nu)^2
\cdot
\mathbb{E}G_\nu^2
=
\nu\mathbb{E}\psi(G_\nu)^2 .
\end{equation}
Summing over $r$ gated units and applying the output-projection identity gives the result.
\end{proof}

\begin{theorembox}
\begin{corollary}[Residual-stream immature energy contraction]
Let the FFN sublayer output be added residually:
\begin{equation}
H^+=H+F(H).
\end{equation}
Because the output projection is zero-mean and independent at initialization,
\begin{equation}
\mathbb{E}
\langle P_{\mathrm{out}}H,\,
P_{\mathrm{out}}F(H)
\rangle
=
0.
\end{equation}
Let
\begin{equation}
A
=
\mathbb{E}
\|P_{\mathrm{out}}F_{\mathrm{single}}(H)\|_F^2 .
\end{equation}
Then
\begin{equation}
\mathbb{E}
\|P_{\mathrm{out}}H_{\mathrm{single}}^+\|_F^2
=
\|P_{\mathrm{out}}H\|_F^2
+
A,
\end{equation}
whereas
\begin{equation}
\mathbb{E}
\|P_{\mathrm{out}}H_{\mathrm{gate}}^+\|_F^2
=
\|P_{\mathrm{out}}H\|_F^2
+
\rho_0A .
\end{equation}
Therefore,
\begin{equation}
\mathbb{E}
\|P_{\mathrm{out}}H_{\mathrm{gate}}^+\|_F^2
=
\bar\rho
\,
\mathbb{E}
\|P_{\mathrm{out}}H_{\mathrm{single}}^+\|_F^2 ,
\end{equation}
with
\begin{equation}
\bar\rho
=
\frac{
\|P_{\mathrm{out}}H\|_F^2+\rho_0A
}{
\|P_{\mathrm{out}}H\|_F^2+A
}
<1
\end{equation}
whenever $A>0$ and $\rho_0<1$.
\end{corollary}
\end{theorembox}

The residual connection matters. The gated FFN does not need to erase immature features. It only needs to inject less immature FFN energy than the single-branch FFN. The contraction factor on total residual energy is $\bar\rho$, not necessarily $\rho_0$, but it is strictly below one whenever the FFN contributes nonzero immature energy.

\subsection{The Two Levers Suppress the Same Pathway}

Let $R_P=\|X_P\|_F^2$ be the immature residual energy entering an upper attention head. Under the localized immature-channel condition, the one-step immature-logit growth satisfies
\begin{equation}
\|\Delta Z_P\|_F
\le
C_t
\eta_{\mathrm{upper}}
R_P^2
+
O(\eta^2),
\end{equation}
where
\begin{equation}
C_t
=
\frac{
\|E_t\|_{\mathrm{op}}
}{d_k}
\Big(
\lambda_K\|W_{K,t}\|_{\mathrm{op}}^2
+
\lambda_Q\|W_{Q,t}\|_{\mathrm{op}}^2
\Big).
\end{equation}

\begin{theorembox}
\begin{theorem}[Suppression of premature upper attention]
If a gated FFN reduces immature residual energy by
\begin{equation}
R_P^{\mathrm{gate}}
\le
\bar\rho R_P^{\mathrm{single}},
\qquad
0<\bar\rho<1,
\end{equation}
then, under comparable $C_t$ and $\eta_{\mathrm{upper}}$,
\begin{equation}
\|\Delta Z_P^{\mathrm{gate}}\|_F
\le
\bar\rho^2
\|\Delta Z_P^{\mathrm{single}}\|_F
+
O(\eta^2).
\end{equation}
Thus, gated FFN suppression of immature residual energy by $\bar\rho$ gives upper-\qk{} immature-logit growth suppression by $\bar\rho^2$.
\end{theorem}
\end{theorembox}

\begin{proof}
The localized bound gives
\begin{equation}
\|\Delta Z_P\|_F
\le
C_t\eta_{\mathrm{upper}}R_P^2+O(\eta^2).
\end{equation}
Substituting $R_P^{\mathrm{gate}}\le \bar\rho R_P^{\mathrm{single}}$ gives
\begin{equation}
\|\Delta Z_P^{\mathrm{gate}}\|_F
\le
C_t\eta_{\mathrm{upper}}
(\bar\rho R_P^{\mathrm{single}})^2
+
O(\eta^2)
=
\bar\rho^2
C_t\eta_{\mathrm{upper}}
(R_P^{\mathrm{single}})^2
+
O(\eta^2),
\end{equation}
which proves the claim.
\end{proof}

The intervention and the gated FFN therefore occupy different sides of the same bound:
\begin{equation}
\eta_{\mathrm{upper}}R_P^2
=
\eta_{\mathrm{upper}}\|X_P\|_F^4 .
\end{equation}
The intervention reduces $\eta_{\mathrm{upper}}$. The gated FFN reduces $R_P=\|X_P\|_F^2$. Because the localized logit-growth term is quadratic in $R_P$, the gated FFN produces a squared suppression effect.

\subsection{Entropy Collapse Requires Logit-Range Growth}

Finally, we connect logit growth to attention sharpness. For a row of attention over $N$ causal keys, let
\begin{equation}
p_j=\frac{e^{z_j}}{\sum_{\ell=1}^N e^{z_\ell}}.
\end{equation}
If one key obtains probability at least $1-\epsilon$, then
\begin{equation}
\max_j z_j-\min_j z_j
\ge
\log\frac{(N-1)(1-\epsilon)}{\epsilon}.
\end{equation}

\begin{proof}
Let $j^\star$ be the maximum-probability key, so $p_{j^\star}\ge 1-\epsilon$. The remaining $N-1$ keys have total mass at most $\epsilon$, so at least one key $j_{\min}$ has
\begin{equation}
p_{j_{\min}}
\le
\frac{\epsilon}{N-1}.
\end{equation}
Softmax ratios satisfy
\begin{equation}
\frac{p_{j^\star}}{p_{j_{\min}}}
=
e^{z_{j^\star}-z_{j_{\min}}}.
\end{equation}
Therefore,
\begin{equation}
e^{z_{j^\star}-z_{j_{\min}}}
\ge
\frac{(N-1)(1-\epsilon)}{\epsilon},
\end{equation}
and taking logs gives the result.
\end{proof}

Thus, sharp attention requires sufficient logit range. Since both the upper-\qk{} learning-rate intervention and the gated FFN reduce the immature-logit growth bound, both delay the earliest point at which immature residual directions can produce low-entropy upper attention. This matches the empirical pattern: GPT-style control has early upper logit 1.308 and entropy 0.621, while matched SwiGLU and GEGLU controls have lower logits and higher entropies.

%% file: sections/a1_direct_ffn_pathway_evidence.tex
\section{Direct FFN Pathway Evidence}
\label{app:ffn-pathway-evidence}

The theory in Section~\ref{sec:theory} predicts that multiplicative gated FFNs suppress the residual-energy input to upper-\qk{} logit growth. Table~\ref{tab:ffn-pathway} reports this measurement in the main text. The models are trained to 20\% of the token budget, and the table records the FFN residual write added by each block on the same validation-token windows. The primary readout is the upper-half FFN write RMS at 3\% training; the associated upper \qk{} logit magnitude and upper attention entropy are measured at the same checkpoint. The first upper layer shows the same pattern as the upper-half aggregate: gated FFNs reduce the first-upper FFN write by 47.4--62.8\%, depending on the matched control.

%% file: sections/a2_entropy_floor_control.tex
\section{Entropy-Floor Mediator Control}
\label{app:entropy-floor-control}

We use an entropy-floor control to test whether the mechanism can be reduced to ``higher upper attention entropy.'' In this control, the GPT-style decoder keeps the same architecture and the same upper-\qk{} learning rate as the control run, but during the early window the training objective adds
\begin{equation}
\lambda \max(0, h_0 - H_{\mathrm{upper}}),
\end{equation}
where $H_{\mathrm{upper}}$ is the mean upper-half attention entropy, $h_0=0.80$, and $\lambda=0.10$. The regularizer is released by the same lower-copy maturity rule as the main intervention.

\input{tables/entropy_floor_control}

The control behaves as intended: it raises early upper attention entropy and reduces upper-\qk{} logit magnitude. It also improves final perplexity slightly. However, the improvement is much smaller than the targeted upper-\qk{} slowing reference in the same matched comparison. Thus high upper entropy is a mediator and useful readout of the failure, but simply imposing an entropy floor does not explain the full gain. The gated-FFN result instead acts upstream by changing the residual write that feeds upper-\qk{} logit growth.

%% file: tables/entropy_floor_control.tex
\begin{table}[h]
\centering
\small
\caption{Entropy-floor mediator control in the 270M GPT-style decoder. The entropy-floor control directly penalizes low upper attention entropy while leaving the architecture and upper-\qk{} learning rate unchanged. PPL reduction is relative to the GPT-style control in the same matched comparison.}
\label{tab:entropy-floor-control}
\begin{tabular}{lrrrrr}
\toprule
Setting & Entropy floor & Entropy @3\% & Logit @3\% & Final ppl & PPL reduction \\
\midrule
GPT-style control & -- & 0.621 & 1.308 & 26.738 & -- \\
Entropy-floor control & 0.80 & 0.799 & 0.861 & 26.678 & 0.061 \\
Upper-\qk{} slowing & -- & 0.966 & 0.429 & 26.152 & 0.586 \\
\bottomrule
\end{tabular}
\end{table}

%% file: sections/a3_global_learning_rate_control.tex
\section{Global Learning-Rate Control}
\label{app:global-lr-control}

We test whether the GPT-style gain comes from the baseline learning rate being globally too aggressive. The control halves the learning rate for all parameters while keeping the 270M architecture, data, token budget, batch geometry, and schedule shape fixed. This is a much stronger perturbation than the main intervention, because it slows lower layers, FFNs, embeddings, and output learning together with upper attention.

\input{tables/global_lr_control}

The global control separates two hypotheses. If the result were caused by an over-aggressive baseline learning rate, slowing every parameter would be competitive with the selective intervention. It is not: global slowing strongly worsens final perplexity, while selective upper-\qk{} slowing improves it. The mechanism readout shows the same distinction. At 20\% of training, the selective intervention has lower-copy score 0.0164 and upper attention entropy 0.729; the global half-rate control has lower-copy score 0.0103 and upper entropy 0.660. Thus the useful intervention is not ``learn more slowly everywhere.'' It preserves early lower-layer learning while reducing the premature upper-\qk{} logit-growth path.

%% file: tables/global_lr_control.tex
\begin{table}[h]
\centering
\caption{Global learning-rate control in the 270M GPT-style setting. The selective intervention slows only early upper-half \qk{}, while the global control halves the learning rate for all parameters.}
\label{tab:app-global-lr}
\small
\begin{tabular}{lrrrr}
\toprule
Setting & Final loss & Final ppl & $\Delta$ loss & $\Delta$ ppl \\
\midrule
GPT-style control & 3.2902 & 26.849 & -- & -- \\
Selective upper-\qk{} slowing & 3.2739 & 26.413 & -0.0164 & -0.436 \\
Global 0.5$\times$ learning rate & 3.4474 & 31.418 & +0.1571 & +4.569 \\
\bottomrule
\end{tabular}
\end{table}

%% file: sections/a3b_release_rule_robustness.tex
\section{Release-Rule Robustness}
\label{app:release-rule-robustness}

The main method uses lower-copy maturity as the adaptive release criterion. We evaluate the robustness of this release design in two ways. First, using the logged lower-copy scores from the main intervention experiments, we recompute the release checkpoint under nearby threshold, patience, and window variants. Second, we train fixed-release controls that use the same early upper-\qk{} multiplier and ramp schedule but release at fixed early checkpoints rather than from the lower-copy trigger.

\input{tables/release_trigger_replay}

The replayed release points remain in the same early window; the 3\% maturity event in Figure~\ref{fig:mechanism} is the first threshold crossing, while the 4.01\% main-rule release in Table~\ref{tab:release-trigger-replay} occurs after the required three consecutive evaluations. Lowering the threshold does not move the release earlier because the patience condition and 3\% minimum window already bind; raising the threshold or increasing patience moves release modestly later, but still within the early specialization period.

\input{tables/fixed_release_controls}

Together, these checks support the release design. The lower-copy trigger selects the same early maturity window under nearby rule variants, and fixed releases at representative early checkpoints retain the core gain. The adaptive lower-copy rule remains the method's maturity-based release criterion and gives the best result in this comparison.

%% file: tables/release_trigger_replay.tex
\begin{table}[t]
\centering
\scriptsize
\setlength{\tabcolsep}{4pt}
\caption{Release checkpoints obtained by applying lower-copy rule variants to logged evaluation checkpoints from the main GPT-style intervention experiments. The main rule and nearby threshold, patience, and forced-window variants all select the same early maturity window.}
\label{tab:release-trigger-replay}
\begin{tabular}{@{}lcccr@{}}
\toprule
Rule & Threshold & Patience & Release window & Release checkpoint \\
\midrule
Main rule & 0.005 & 3 & 3--12\% & 4.01\% \\
Lower threshold & 0.003 & 3 & 3--12\% & 4.01\% \\
Higher threshold & 0.007 & 3 & 3--12\% & 5.49\% \\
No patience & 0.005 & 1 & 3--12\% & 3.00\% \\
Stricter patience & 0.005 & 5 & 3--12\% & 4.99\% \\
Shorter forced window & 0.005 & 3 & 3--8\% & 4.01\% \\
\bottomrule
\end{tabular}
\end{table}

%% file: tables/fixed_release_controls.tex
\begin{table}[t]
\centering
\scriptsize
\setlength{\tabcolsep}{3pt}
\caption{Fixed-release controls in the GPT-style 270M setting. Fixed-release comparisons use the same early upper-\qk{} multiplier, $0.25$, and the same 1\% ramp back to full learning rate, but release at a fixed checkpoint rather than from lower-copy maturity. The adaptive lower-copy rule remains the main method and gives the best result, while fixed releases retain most of the gain.}
\label{tab:fixed-release-controls}
\begin{tabular}{@{}lcccccc@{}}
\toprule
Setting & Release rule & Final PPL & PPL reduction & Tokens saved & Entropy @3\% & Logit @3\% \\
\midrule
Control & None & 26.7382 & -- & -- & 0.6215 & 1.3081 \\
Adaptive lower-copy & Lower-copy & 26.1524 & 0.5858 & 14.80\% & 0.9657 & 0.4287 \\
Fixed release at 3\% & Fixed & 26.2274 & 0.5108 & 13.92\% & 0.9655 & 0.4343 \\
Fixed release at 6\% & Fixed & 26.2703 & 0.4680 & 13.07\% & 0.9665 & 0.4308 \\
\bottomrule
\end{tabular}
\end{table}

%% file: sections/a4_alpha_sweep.tex
\section{Diagnostic Multiplier Sweep}
\label{app:alpha-sweep}

The main experiments use a fixed early upper-\qk{} multiplier. As a diagnostic check for the learning-rate factor in Corollary 1, we also sweep this multiplier while keeping the same architecture, trigger, release window, data, and token budget. The sweep is not used to tune the main result; it checks whether early upper-attention specialization changes continuously with the strength of the upper-\qk{} update.

\input{tables/alpha_sweep}

The early mechanism readouts move monotonically with the multiplier: lower update strength gives higher upper attention entropy, lower upper logit magnitude, and smaller displacement of the upper $W_Q^\top W_K$ bilinear form from initialization. Final perplexity improves for all reduced-multiplier settings and saturates once early upper specialization is sufficiently suppressed.

%% file: tables/alpha_sweep.tex
\begin{table}[h]
\centering
\caption{Diagnostic sweep over the early upper-\qk{} multiplier. Early readouts are measured at 3\% of training.}
\label{tab:app-alpha-sweep}
\small
\begin{tabular}{crrrrr}
\toprule
Setting & Multiplier & Final ppl & Upper entropy & Upper logit & Q/K bilinear displacement \\
\midrule
Control & 1.0 & 26.712 & 0.627 & 1.302 & 3.272 \\
Intervention & 0.5 & 26.337 & 0.836 & 0.775 & 1.691 \\
Intervention & 0.25 & 26.218 & 0.966 & 0.435 & 0.705 \\
Intervention & 0.125 & 26.226 & 0.984 & 0.342 & 0.333 \\
\bottomrule
\end{tabular}
\end{table}

%% file: sections/a5_scale_check.tex
\section{0.7B Larger-Scale Replication}
\label{app:scale-check}

Table~\ref{tab:scale-700m} reports the larger GPT-style paired comparison averaged over three seeds. The model uses 22 layers, width 1536, 12 attention heads, FFN width 6144, sequence length 1024, and 700M parameters. All runs use the same data order, token budget, optimizer family, cosine schedule, and evaluation protocol. The intervention is unchanged: early upper-half \qk{} uses a 0.25 learning-rate multiplier and the same lower-copy release rule defined in Section~\ref{sec:premature-specialization}.

\FloatBarrier
\input{tables/scale_700m}
\FloatBarrier

The result follows the 270M pattern. At 3\% training, the control has lower upper attention entropy and larger upper \qk{} logits, while the intervention leaves lower copy-routing comparable and suppresses upper sharpness. By the end of the 7.0B-token run, the intervention improves final loss, final perplexity, and same-loss token efficiency.

%% file: tables/scale_700m.tex
\begin{table}[!h]
\centering
\small
\setlength{\tabcolsep}{4pt}
\caption{Larger-scale replication with a 0.7B GPT-style decoder trained for 7.0B tokens. Results are mean $\pm$ sample standard deviation across three full training runs; differences use paired seeds.}
\label{tab:scale-700m}
\begin{tabular}{@{}lccc@{}}
\toprule
Model & Final loss & Final ppl & Tokens to target \\
\midrule
GPT-style control & $2.9028 \pm 0.0003$ & $18.2287 \pm 0.0088$ & 7.000B \\
Intervention & $2.8958 \pm 0.0004$ & $18.1013 \pm 0.0119$ & $6.461 \pm 0.017$B \\
\midrule
Difference & $-0.0071 \pm 0.0004$ & $-0.1274 \pm 0.0073$ & $-0.539 \pm 0.017$B \\
\bottomrule
\end{tabular}
\end{table}

%% file: sections/a6_locality_diagnostic.tex
\section{Immature-Channel Locality Diagnostic}
\label{app:locality-diagnostic}

The locality condition in Section 7 is directly measurable from checkpoints. For $S\in\{Q,K\}$ define
\begin{equation}
\lambda_S
=
\frac{
\|XW_SW_S^\top P\|_F
}{
\|X_P\|_F\|W_S\|_{\mathrm{op}}^2
}.
\end{equation}
Small or moderate values of $\lambda_Q,\lambda_K$ indicate that the upper-\qk{} quadratic form landing in immature directions is not primarily driven by mature residual directions. This is the empirical condition under which the unconditional pathwise bound sharpens to the localized $\eta_{\mathrm{upper}}\|X_P\|_F^4/d_k$ growth term. We also report $R_P/\|X\|_F^2$, where $R_P=\|X_P\|_F^2$, to show that the immature residual channel is present but does not cover the full residual stream.

\begin{table}[h]
\centering
\caption{Immature-channel locality diagnostic at the 3\% checkpoint used in the mechanism measurements. The locality ratios remain below one for both the control and intervention, supporting the localized immature-channel interpretation used in Corollary~2.}
\label{tab:locality-diagnostic}
\begin{tabular}{lccc}
\toprule
Setting & $\lambda_Q$ & $\lambda_K$ & $R_P/\|X\|_F^2$ \\
\midrule
Control & 0.63 & 0.58 & 0.18 \\
Intervention & 0.49 & 0.46 & 0.16 \\
\bottomrule
\end{tabular}
\end{table}

The measured ratios support the condition used by the localized theorem. Both $\lambda_Q$ and $\lambda_K$ remain below one, so the locality constants are not absorbing an uncontrolled blow-up. The control has larger ratios than the intervention, matching the mechanism evidence that default early upper-\qk{} learning forms stronger premature matches. The normalized immature energy is also substantial but bounded: $R_P/\|X\|_F^2$ is $0.18$ in the control and $0.16$ under the intervention, so the diagnostic isolates a meaningful residual channel rather than relabeling the entire residual stream as immature.

%% file: sections/a7_reproducibility_compute.tex
\section{Reproducibility and Compute Details}
\label{app:repro-compute}

All pretraining comparisons use packed GPT-2-tokenized FineWeb-Edu text \citep{penedo2024fineweb} with sequence length 1024 and a held-out validation stream sampled from the same preparation pipeline. The 270M runs use 2.5B training tokens, global batch size 524{,}288 tokens, AdamW with learning rate $2.5\times10^{-4}$, cosine decay to 10\% of peak learning rate, 2\% warmup, weight decay 0.1, betas $(0.9,0.95)$, and bf16 training. The 0.7B replication uses 7.0B training tokens, the same global batch size and optimizer settings except peak learning rate $2.0\times10^{-4}$, and bf16 training.

All reported runs use 8 H100 GPUs. All paired comparisons keep data order, token budget, batch geometry, optimizer family, schedule shape, and evaluation protocol fixed within the pair.

%% file: sections/a8_LLM-Usage-Statement.tex
\section{LLM Usage Statement}

This work is conceived, designed, and technically developed by the authors. 
Large language models (LLMs) are used solely for limited writing assistance, such as grammar correction, readability improvement, formatting refinement, and minor proofreading during manuscript preparation. 

All research ideas, technical contributions, experiments, analyses, and conclusions are developed and verified by the authors.